\documentclass{article} % For LaTeX2e
\usepackage{iclr2025_conference,times}

\usepackage{amsmath,amsfonts,bm}
\usepackage{xcolor}
\usepackage{amssymb}
\usepackage{mathtools}
\usepackage{amsthm}
\usepackage{colortbl}

\def\eqref#1{equation~\ref{#1}}
\def\1{\bm{1}}

\newcommand{\tcen}[1]{\multicolumn{1}{c}{#1}}

\definecolor{royalblue}{rgb}{0.06, 0.75, 0.99}

\usepackage{xspace}

 \newtheorem{lemma}{Lemma}
 
 \newtheorem{proposition}{Proposition}
 
\definecolor{mygray}{gray}{0.95}

\def\eqref#1{\ref{#1}}
\def\1{\bm{1}}

\DeclareMathAlphabet{\mathsfit}{\encodingdefault}{\sfdefault}{m}{sl}
\SetMathAlphabet{\mathsfit}{bold}{\encodingdefault}{\sfdefault}{bx}{n}

\newcommand{\E}{\mathbb{E}}
\newcommand{\cU}{{\mathcal U}}
\newcommand{\bbE}{{\mathbb E}}

\newcommand{\EE}{{\bbE}}

\newcommand{\R}{\mathbb{R}}

\usepackage[utf8]{inputenc} % allow utf-8 input
\usepackage[T1]{fontenc}    % use 8-bit T1 fonts
\usepackage{url}            % simple URL typesetting
\usepackage{booktabs}       % professional-quality tables
\usepackage{amsfonts}       % blackboard math symbols
\usepackage{nicefrac}       % compact symbols for 1/2, etc.
\usepackage{microtype}      % microtypography
\usepackage{xcolor}
\definecolor{linkcolor}{RGB}{74, 102, 146}
\usepackage[colorlinks=true,allcolors=linkcolor,pageanchor=true,plainpages=false,pdfpagelabels,bookmarks,bookmarksnumbered]{hyperref}
\usepackage{cleveref}
\usepackage{lipsum} 
\usepackage{enumitem}

\usepackage{algorithm}
\usepackage{algpseudocode}

\usepackage{hyperref}
\usepackage{url}

\renewcommand{\eqref}[1]{(\ref{#1})}

\title{An Efficient On-Policy Deep Learning Framework for Stochastic Optimal Control}

\author{Mengjian Hua \\
NYU-ECNU Institute of Mathematical Sciences\\
NYU Shanghai\\
Shanghai, China 200124\\
\texttt{mh5113@nyu.edu} \\
\And
Mathieu Lauri{\`e}re\\
Shanghai Frontiers Science Center of AI and DL\\
NYU-ECNU Institute of Mathematical Sciences\\
NYU Shanghai\\
Shanghai, China 200124\\
\texttt{mathieu.lauriere@nyu.edu} \\
\And
Eric Vanden-Eijnden   \\
Courant Institute of Mathematical Sciences \\
New York University \\
New York, NY 10012, USA \\
\texttt{eve2@cims.nyu.edu} \\
}
\iclrfinalcopy % Uncomment for camera-ready version, but NOT for submission.

\begin{document}

\maketitle

\vspace{-.5cm}
\begin{abstract}
We present a novel on-policy algorithm for solving stochastic optimal control (SOC) problems. By leveraging the Girsanov theorem, our method directly computes on-policy gradients of the SOC objective without expensive backpropagation through stochastic differential equations or adjoint problem solutions. This approach significantly accelerates the optimization of neural network control policies while scaling efficiently to high-dimensional problems and long time horizons. We evaluate our method on classical SOC benchmarks as well as applications to sampling from unnormalized distributions via Schr\"odinger-F\"ollmer processes and fine-tuning pre-trained diffusion models. Experimental results demonstrate substantial improvements in both computational speed and memory efficiency compared to existing approaches.
\end{abstract}

\section{Introduction}

Stochastic Optimal Control (SOC) problems~\citep{mortensen1989stochastic,fleming2012deterministic} arise across sciences and engineering, from traditional domains like finance and economics~\citep{pham2009continuous,FLEMING2004979,aghion1992model} and robotics~\citep{theodorou2011iterative,pavlov2018narrow} to emerging applications in sampling complex distributions and simulating rare events~\citep{zhang2022path,holdijk2023stochasticoptimalcontrolcollective,hartmann2013characterization,HartmannCarsten,ribeira2024improving}. Their goal is to optimize a cost function by adding an adjustable drift (the control) to a reference stochastic differential equation (SDE).

While low-dimensional SOC problems can be solved using standard numerical methods for the Hamilton-Jacobi-Bellman equation, these approaches fail in high dimensions. This has motivated recent deep learning (DL) solutions~\citep{han2016deep,han2018solving,hure2020deep,domingo2023stochastic,germain2021neural,hu2024recent} that parameterize the control using neural networks and optimize it via stochastic gradient descent on the SOC objective, evaluated using controlled SDE solutions. This Neural SDE approach~\citep{tzen2019neural,li2020scalable}, though conceptually simple, requires differentiating through SDE solutions—making it computationally expensive and limiting scalability.

One alternative uses the Girsanov theorem to compute the SOC objective via expectations over a reference process with a control independent from the one being optimized. However, this introduces an exponential weighing factor whose high variance (when the reference control differs significantly from the actual control) again limits scalability.

We propose a new approach based on expressing the gradient of the SOC objective exactly through expectations over the controlled process (on-policy evaluation) without differentiating through process solutions (simulation-free). This result, first derived by~\citet{yang1991monte}, avoids exponential weighing factors. We show this gradient can be computed via automatic differentiation of an alternative objective by selectively detaching parameters from the computational graph.

Specifically, our \textbf{main contributions} are:
\begin{itemize}[leftmargin=0.15in,noitemsep,topsep=0pt]%\itemsep-0.2em 
\item We propose an on-policy algorithm for solving generic SOC problems via deep learning that scales to scenarios where Neural SDE methods are computationally intractable. Our approach uses on-policy evaluation of gradients without requiring differentiation through SDE solutions.
\item We show how to apply our approach to construct F\"ollmer processes between a point mass and an unnormalized target distribution, enabling sampling from this target and computing its normalization constant.
\item We also show how to fine-tune diffusion-based generative models by solving an SOC problem that transforms an initial sampling SDE into one that samples from a tilted distribution weighted by a reward function.
\item Through numerical experiments, we demonstrate significant reductions in computational time and memory usage compared to methods requiring SDE solution differentiation, such as Neural SDE approaches.
\end{itemize}

\subsection{Related Work}
\label{sec:related}

Deep learning approaches to SOC have evolved along several directions. \citet{han2016deep} pioneered learning feedback control functions for high-dimensional problems, inspiring algorithms for backward SDEs and PDEs~\citep{han2018solving}. Alternative approaches include the algorithms proposed in~\citet{ji2020three}, dynamic programming methods~\citep{hure2021deep,hure2020deep}, and stochastic optimal control matching~\citep{domingo2023stochastic} based on iterative diffusion optimization~\citep{nüsken2023solvinghighdimensionalhamiltonjacobibellmanpdes}. These methods either use off-policy learning with high-variance estimators or require costly differentiation through SDE solutions.

The gradient formula we use was originally proposed in~\citet{yang1991monte} in the context of sensitivity analysis in finance~\citep{pham2009continuous,FLEMING2004979,aghion1992model}, with generalizations in~\citet{gobet2005sensitivity}. While referenced in recent DL works~\citep{mohamed2020monte,li2020scalable,lie2021frechet,domingo2023stochastic,ribeira2024improving,domingoenrich2024taxonomy}, it has not been fully exploited algorithmically. Our approach also connects to Reinforcement Learning~\citep{quer2024connecting,domingo2024adjointmatching,domingoenrich2024taxonomy}, resembling a continuous-time version of the REINFORCE algorithm~\citep{williams1992simple,williams1988towards,sutton1999policy}.

For F\"ollmer processes~\citep{follmer1986time}, several deep learning methods have been proposed~\citep{huang2021schrodinger,jiao2021convergence,vargas2023bayesian}. The Path Integral Sampler (PIS)~\citep{zhang2022path} is closest to our approach, as it performs on-policy minimization of the same SOC objective. However, PIS requires differentiating through controlled processes, while our approach avoids this costly step.

Generative models based on diffusion can be fine-tuned by modifying their drifts based on reward functions~\citep{fan2024reinforcement,clark2024directly,uehara2024finetuning}. This task can be formulated as a SOC problem~\citep{domingo2024adjointmatching}. Our method offers an efficient solution when the base distribution is a point-mass.

\section{Methods}
\label{sec:method}

\subsection{Problem setup} 
\label{sec:setup}

We consider the stochastic optimal control (SOC) problem:
\begin{equation}
    \label{eq:obj}
    \min_{u \in \mathcal{U}} J(u)\qquad \text{with}\qquad  J(u) = \mathbb{E} [\mathcal{J}(u,X^u)], 
\end{equation}
in which, for a generic process $X = (X_t)_{t \in [0,T]}$, 
\begin{equation}
\label{eq:obj:jj}
    \mathcal{J}(u,X) = \int_{0}^T \left(\tfrac{1}{2} | u_t(X_t)|^2 + f_t(X_t)  \right) dt + g(X_T),
\end{equation}
and $X^u=(X_t^u)_{t\in[0,T]}$ is the solution to the SDE
\begin{equation}
\label{eq:sde}
    d X_t^u = \left( b_t(X_t^u) + \sigma_t  u_t(X_t^u) \right) dt +  \sigma_t d W_t, \quad X_0^u \sim \mu_0.
\end{equation}
In these equations,  $X^u_t\in \R^d$ is the system state, $\E$ denotes expectation over the law of $X^u$, $u : [0,T]\times \R^d \to \R^d$ is a closed-loop Markovian control that belongs to some set $\mathcal{U}$ of admissible controls to be specified later, $f: [0,T]\times \R^d \to \R^d$ is the state cost,  $g:\R^d \to \R^d$ is the terminal cost, $b:[0,T]\times \R^d \to \R^d$ is the base drift,  $\sigma : [0, T ] \to\R^d\times \R^d$ is the  volatility matrix, which we assume invertible and independent of the state~$X^u_t$, $(W_t)_{t\in [0,T]}$ is a Wiener process taking values in~$\R^d$, and $\mu_0$ is some probability distribution on $\R^d$ for the initial state. 

We are interested in solving~\eqref{eq:obj} in situations where the set of admissible controls $\mathcal{U}$ is a rich parametric class, for example made of deep neural networks (DNN). We denote functions in the class by $u^\theta$, where $\theta\in \Theta$ collectively denotes the parameters to be adjusted, e.g. the weights if we use a DNN. 

\subsection{Reformulation with Girsanov theorem}
\label{sec:grisa}

The key challenge in solving the SOC problem~\eqref{eq:obj} is that a vanilla calculation of the objective's gradient requires to differentiate $X^u$ since it depends on the control $u$:  this requires costly back-propagation through the solutions of the SDE~\eqref{eq:sde}. We propose a method to compute the gradient of~\eqref{eq:obj} that avoids this expensive step and refer to it as a \emph{simulation-free} method. 

To this end, we first use the Girsanov theorem to reformulate the problem so that the control $u$ appears explicitly in the objective:
\begin{lemma}
    \label{th:girs}
    Given a reference control $v \in \cU$, the objective in~\eqref{eq:obj} can be expressed as
\begin{equation}
\label{eq:obj:G}
    \begin{aligned}
        J(u) &= \mathbb{E}\left[\mathcal{J}(u,X^v) M (u,v)\right],
    \end{aligned}
\end{equation}
where $X^v=(X^v_t)_{t\in[0,T]}$ solves the SDE~\eqref{eq:sde} with $u$ replaced by $v$, $M(u,v)$ is the Girsanov factor
\begin{equation}
\label{eq:gfact}
\begin{aligned}
    M (u,v) =  \exp\Biggl( &- \int_{0}^T \left(v_t(X_t^v) - u_t(X_t^v)\right) \cdot d W_t - \frac{1}{2} \int_{0}^T \left|v_t(X_t^v) - u_t(X_t^v)\right|^2 dt  \Biggr),
\end{aligned}
\end{equation}
and the expectation $\E$ in~\eqref{eq:obj:G} is taken over the law of $X^v$.
\end{lemma}

The result follows directly from the Girsanov change of measure formula between the law of $X^u$ and $X^v$. For a proof, see e.g. \citet{karatzas1991brownian}.

Expression~\eqref{eq:obj:G} presents $J(u)$ as an off-policy objective, making $u$ explicit since the process $X^v$ is independent of this control. This eliminates the need to differentiate through state process trajectories when computing gradients. However, empirically evaluating $J(u)$ and its gradient using~\eqref{eq:obj:G} with finite samples from $X^v$ yields an estimator whose variance depends heavily on the reference control $v$. This suggests keeping $v$ close to $u$. Next we show that we can use~\eqref{eq:obj:G} to evaluate the gradient using the controlled process itself.

\subsection{Gradient computation} 
\label{sec:grad}

Our method builds on a gradient formula for parametrized controls $u=u^\theta$, originally derived in~\citet{yang1991monte} and also presented in~\citet{ribeira2024improving}:

\begin{proposition}
\label{th:grad}
Let $u^\theta$ with $\theta\in \Theta$ be a parametric realization of a control in $\mathcal{U}$ and denote $L(\theta) \equiv J(u^\theta)$ the objective~\eqref{eq:obj} viewed as a function of $\theta$. Then 
\begin{equation}
\label{eq:grad}
\begin{aligned}
\partial_\theta L(\theta)
    = \EE\left[ \int_0^T u_t^\theta(X^{\theta}_t) \cdot \partial_\theta u_t^\theta(X^{\theta}_t) dt\right]  
     + \EE\left[ \mathcal{J}(u^\theta,X^\theta) \int_0^T \partial_\theta u_t^\theta(X^{\theta}_t)\cdot dW_t \right]
\end{aligned}
\end{equation}
where $\partial_\theta u_t^\theta(X^{\theta}_t)$ denotes $\partial_\theta u_t^\theta(x)$ evaluated at $x=X^{\theta}_t$ and  $X^\theta=(X^\theta_t)_{t\in [0,T]} \equiv (X^{u^\theta}_t)_{t\in [0,T]}$ solves the SDE
\begin{equation}
\label{eq:sde:theta}
    d X_t^\theta = \left( b_t(X_t^\theta) + \sigma_t u^\theta_t(X_t^\theta) \right) dt +  \sigma_t  d W_t, \quad X^\theta_0 \sim \mu_0,
\end{equation}
and the expectation $\E$ in~\eqref{eq:grad} is taken over the law of $X^\theta$.
\end{proposition}
For completeness, we give the proof of this proposition in Appendix~\ref{app:finetuning}.
Expression~\eqref{eq:grad} eliminates the need to differentiate through state trajectories $(X^\theta_t)_{t\in[0,T]}$. While it requires an invertible, control-independent volatility $\sigma_t$, the formula can be extended to control-dependent volatilities using Malliavin calculus~\citep{gobet2005sensitivity}.

\begin{algorithm}[tb]
\caption{Simulation-Free On-Policy Training}
\begin{algorithmic}[1]
\State \textbf{Initialize:} $n$ walkers, $K$ time steps, model parameters $\theta$ for $ u^\theta$, gradient descent optimizer

\Repeat
\State Set $\bar \theta = \text{stopgrad}(\theta)$
\State Randomize time grid: $t_1, \dots, t_K \sim \text{Uniform}(0, T)$
\State Add $t_0=0$, $t_{K}=T$, and sort such that $0=t_0 < t_1 < \dots < t_{K-1}< t_{K} =T$
\State Set $\Delta t_k = t_{k+1} - t_k$
\For{each walker $i = 1, \dots, n$} 
 \State Set $x^i_{0} \sim \mu_0$
\For{$k=0, \dots, K-1$}
        \State $\Delta W^i_k = \sqrt{\Delta t_k} \, \zeta^i_k$, where $\zeta^i_k \sim N(0,\text{Id})$
        \State $x^i_{t_{k+1}} = x^i_{t_k} + u^{\bar \theta}_{t_k} (x^i_{t_k}) \Delta t_k  + \sigma_{t_k}\Delta W^i_k $
    \EndFor
 %   \State Compute:
\State $A^i_K =   \sum_{k=1}^K \tfrac12 |u^\theta_{t_k}(x^i_{t_k})|^2\Delta t_k $
        \State $\bar B^i_K =  \sum_{k=1}^K \left(\tfrac12 |u^{\bar \theta}_{t_k}(x^i_{t_k})|^2+f_{t_k}(x^i_{t_k})\right)\Delta t_k$
        \State $C_K^i = \sum_{k=1}^K u_{t_k}^\theta(x^i_{t_k})  \cdot \Delta W^i_k $
        \EndFor
    \State Compute:  $\hat L_n(\theta,\bar\theta) = n^{-1} \sum_{i=1}^n \big[ A_K^i + \big(\bar B^i_K + g(x^i_{t_{K}})\big) C_K^i \big]$.
    \State Compute %the gradient of $\hat L_n(\theta,\bar\theta)$
    $\partial_\theta \hat L(\theta,\bar \theta)\big|_{\bar \theta = \theta}$ and take a step of gradient descent to update $\theta$.
\Until{converged}
\end{algorithmic}
\label{alg:train}
\end{algorithm}

\subsection{Alternative objective for implementation}
\label{sec:impl}

Equation~\eqref{eq:grad} can be implemented to directly estimate the gradient of the objective $L(\theta)=J(u^\theta)$ by estimating the expectation empirically over an ensemble of independent realizations of the SDE~\eqref{eq:sde:theta}. Alternatively, we can use automatic differentiation of an alternative objective~\citep{ribeira2024improving,domingoenrich2024taxonomy}:

\begin{proposition}
    \label{th:altloss}
    We have 
\begin{equation}
    \label{eq:equiv:grad}
    \partial_\theta L(\theta) = \partial_\theta \hat L(\theta,\bar \theta)\big|_{\bar \theta = \theta},
\end{equation}
where  we defined
\begin{equation}
\label{eq:alt:loss}
\begin{aligned}
\hat L(\theta,\bar \theta)
    &= \EE\left[ \int_0^T \tfrac12|u_t^\theta(X^{\bar \theta}_t) |^2 dt\right]  
     + \EE\left[ \mathcal{J}(u^{\bar \theta},X^{\bar \theta}) 
    \int_0^T u_t^\theta(X^{\bar \theta}_t)\cdot dW_t \right]
\end{aligned}
\end{equation}
in which $X^{\bar \theta} =(X^{\bar \theta}_t)_{t\in[0,T]} $ solves \eqref{eq:sde:theta} with $u^\theta$ replaced by $u^{\bar \theta}$ and the expectation $\E$ in~\eqref{eq:grad} is taken over the law of $X^{\bar \theta}$.
\end{proposition}

The proof of this proposition is immediate by direct calculation so we omit it for the sake of brevity.

The gradient $\partial_\theta \hat L(\theta,\bar \theta)\big|_{\bar \theta = \theta}$ can be computed via automatic differentiation by using $\bar\theta=\text{stopgrad}(\theta)$. This avoids differentiating through $\bar X \equiv X^{\bar \theta}$ while maintaining an on-policy objective. The expectation can be estimated empirically using samples from the SDE~\eqref{eq:sde:theta}, as detailed in Algorithm~\ref{alg:train}.

\subsection{Application to sampling via construction of a F\"ollmer process}
\label{sec:follmer}

By definition, the F\"ollmer process that samples a given target probability distribution~$\mu$ is the process $(Y^u_t)_{t\in[0,1]}$ that uses the optimal control $u$ obtained by solving

\begin{equation}
\label{eq:obj:folder}
\min_{u\in \mathcal{U}} \mathbb{E}\int_{0}^1 \tfrac{1}{2} | u_t(Y_t^u)|^2 dt,
\end{equation}
where 
\begin{equation}
\label{eq:sde:Y}
d Y_t^u = u_t(Y_t^u) dt +  d W_t,\quad Y_0^u \sim \delta_0, \quad Y^u_{1} \sim \mu.
\end{equation}
This problem is a special case of the Schr\"odinger bridge problem~\citep{leonard2014survey} when the base distribution is the Dirac delta distribution $\delta_0$, i.e. the point mass at $x=0$. 

While the minimization problem in~\eqref{eq:obj:folder} differs from a standard SOC problem~\eqref{eq:obj} due to its terminal condition $Y^u_{t=1} \sim \mu$, it can be reformulated as a SOC problem~\citep{leonard2014survey,chen2014relation}—an insight exploited by~\citet{zhang2022path}. The connection is as follows:

\begin{proposition}
    \label{th:follmer}
    Assume that $\mu$ is absolutely continuous with respect of the Lebesgue measure and let its probability density function be $\rho(x) = Z^{-1}e^{-U(x)}$ where $U:\R^d \to\R$ is a known potential and $Z=\int_{\R^d} e^{-U(x)} dx <\infty$ is an unknown normalization factor. Consider the SOC problem using the objective
\begin{equation}
\label{eq:obj:folder:soc}
    J(u) = \mathbb{E}\left[\int_{0}^1 \tfrac{1}{2} | u_t(X_t^u)|^2 dt  -\tfrac12 |X^u_1|^2 + U(X^u_1)]\right],
\end{equation}
where $(X_t^u)_{t\in[0,T]}$ solves the SDE
\begin{equation}
    \label{eq:sde:follmer}
    dX^u_t = u_t(X^u_t) dt + dW_t, \qquad X^u_{t=0} \sim \delta_0.
\end{equation}  
Then the process $(X^u_t)_{t\in[0,1]}$ obtained by using the optimal control minimizing~\eqref{eq:obj:folder:soc} in the SDE~\eqref{eq:sde:follmer} is the F\"ollmer process that satisfies $X^u_{t=1} \sim \mu$.
\end{proposition}

We omit the proof of this proposition since it is a special case of Proposition~\ref{th:fine:tune} established below.

Our approach can solve the SOC problem in Proposition~\ref{th:follmer}, providing a simulation-free implementation of the Path Integral Sampler (PIS) \citep{zhang2022path}. Section~\ref{sec:expeiments} demonstrates the computational advantage of our approach through examples.

Since we replaced the terminal constraint in SDE~\eqref{eq:sde:follmer} with a terminal cost in~\eqref{eq:obj:folder:soc}, $X^u_{t=1} \sim \mu$ is not guaranteed for suboptimal controls. However, as noted in~\citet{zhang2022path}, we can still compute unbiased expectations over $\mu$ for any control through Girsanov reweighting: 

\begin{proposition}
    \label{th:girs:follm}
    Consider the process~$(X^u_t)_{t\in[0,T]}$ obtained by solving the SDE~\eqref{eq:sde:follmer} with any (not necessary optimal) control $u$. Then, given any suitable test function $h:\R^d\to\R$, we have
\begin{equation}
    \label{eq:expect}
    \int_{\R^d} h(x) \mu(dx) = Z^{-1} \E\left[h(X^u_T) M(u)\right], \qquad Z = \int_{\R^d} e^{-U(x)} dx = \E \left[M(u)\right],
\end{equation}
where we defined
\begin{equation}
    \label{eq:M:f}
    \begin{aligned}
    M(u) & = (2\pi)^{d/2}\exp\left( - \int_0^1  \tfrac12|u_t(X^u)|^2dt  - \int_0^1 u_t(X^u_t)\cdot dW_t +\tfrac12 |X_1^u|^2- U(X^u_1) \right).
    \end{aligned}
\end{equation}
In addition  $M(u)=Z$ \textit{iff} $u$ is the optimal control minimizing the SOC problem with objective~\eqref{eq:obj:folder:soc}. 
\end{proposition}

We also omit the proof of this proposition since it is a special case of Proposition~\ref{th:fine:tune:2} established below.

\subsection{Application to fine-tuning}
\label{sec:finetuning}

The approach in Section~\ref{sec:follmer} can be adapted to fine-tune generative models. Suppose that the drift $b$ in the SDE
\begin{equation}
    \label{eq:ref:Y}
    d Y_t = b_t(Y_t) dt+ \sigma_t dW_t, \qquad Y_0 \sim \delta_0,
\end{equation}
has been tailored in such a way that $Y_{t=T} \sim \nu$ where $\nu$ is a given probability distribution. Learning such a $b$ can for instance be done using the framework of score-based diffusion models~\citep{song2021scorebasedgenerativemodelingstochastic} or stochastic interpolants~\citep{albergo2022building,albergo2023stochasticinterpolantsunifyingframework} tailored to building F\"ollmer processes~\citep{chen2024forecasting}. Assume that we would like to fine-tune this diffusion so that it samples instead the probability distribution 
\begin{equation}
    \label{eq:nu}
    \mu(dx) = Z^{-1} e^{r(x)} \nu(dx), \quad Z= \int_{\R^d} e^{r(x)} \nu(dx),
\end{equation}
obtained by tilting $\nu$ by the reward function $r:\R^d \to \R$ (assuming that this tilted measure is normalizable, i.e. $Z<\infty$). Such problems arise in the context of image generation where they have received a lot of attention lately. Our next result shows that it can be cast into a SOC problem.

\begin{proposition}
    \label{th:fine:tune}  
    Consider the SOC problem~\eqref{eq:obj} with zero running cost, $f=0$, and terminal cost set to  minus the reward function, $g=-r$, in the objective~\eqref{eq:obj:folder}. Assume also that the drift $b$ and the volatility $\sigma$ used in  the SDE~\eqref{eq:sde} are the same as those used in the SDE~\eqref{eq:ref:Y}  that guarantees that $Y_{t=T} \sim \nu$. Then the solutions of the SDE~\eqref{eq:sde} solved with the optimal $u$ minimizing this SOC problem and $X^u_{t=0} = 0$ are such that $X^{u}_{t=T} \sim \mu$.
\end{proposition}

The proof of this proposition is given in Appendix~\ref{app:finetuning}. Note that  Proposition~\ref{th:follmer} follows from Proposition~\ref{th:fine:tune} as a special case if we set $b_t(x)=0$,  $\sigma_t=1$, and $T=1$, in which case $\nu = N(0,\text{Id})$, and we can set $r(x) = -U(x)+\frac12 |x|^2$ to target $\mu(dx) = Z^{-1} e^{-U(x)} dx$. 

The SOC problem in Proposition~\ref{th:fine:tune} can again be solved in a simulation-free way using our approach, thereby offering a simple alternative to the Adjoint Matching method proposed in~\citet{domingo2024adjointmatching}. Since in practice the learned control will be imperfect, we will need to reweigh the samples to get unbiased estimates of expectations over them. It can be done using this result:
\begin{proposition}
    \label{th:fine:tune:2}
    Let $X^u$ solve SDE~\eqref{eq:sde} starting from the initial condition $X^u_{t=0} = 0$ with an arbitrary (not necessarily optimal) $u$ and with the drift $b$ and the volatility $\sigma$  that guarantee that the solutions the SDE~\eqref{eq:sde} satisfy $Y_{t=T} \sim \nu$. Then given any suitable test function $h:\R^d\to\R$, we have
\begin{equation}
    \label{eq:expect:ft}
    \begin{aligned}
&  \int_{\R^d}  h(x) \mu(dx)  = Z^{-1} \E \left[h(X^u_T) M_r(u)\right], \qquad  Z = \int_{\R^d} e^{r(x)} \nu(dx) = \E \left[M_r(u)\right],
\end{aligned}
\end{equation}
where we defined
\begin{equation}
    \label{eq:M:ft}
    \begin{aligned}
    M_r(u) = & \exp\left( - \int_0^T  \tfrac12|u_t(X^u)|^2dt
    -\int_0^T u_t(X^u_t)\cdot dW_t +r(X^u_T) \right)
    \end{aligned}
\end{equation}
and the expectation is taken over the law of~$X^u$.
In addition,  $M_r(u)=Z$ \textit{iff} $u$ is the optimal control specified in Proposition~\ref{th:fine:tune}. 
\end{proposition}

The proof of this proposition is given in Appendix~\ref{app:finetuning}. Proposition~\ref{th:girs:follm} follows from Proposition~\ref{th:fine:tune:2} as a special case if we set $b_t(x)=0$,  $\sigma_t=1$, $T=1$,  and  $r(x) = -U(x)+\frac12 |x|^2$.

\section{Experiments}
\label{sec:expeiments}

We test our method on two applications: sampling from unnormalized distributions via Schr\"odinger-F\"ollmer processes and fine-tuning pre-trained diffusion models. In Appendix~\ref{app:add_expeiments}, we also report the performance of our approach on classical SOC benchmarks involving linear Ornstein-Ulhenbeck processes with linear and quadratic costs, that are amenable to exact solution for benchmarking.

\subsection{Sampling from an unnormalized distribution}
\label{sec:unnorm}

We test our method on Neal's funnel distribution in $d=10$ dimensions, which can be sampled by solving the SOC problem formulated in Sec.~\ref{sec:follmer}. The distribution is defined by $x_0\sim N(0,\sigma_0)$ and $x_{1:9}|x_0\sim N(0,e^{x_0}\text{Id})$. Previously examined by~\citet{zhang2022path} using the Path Integral Sampler (PIS), this distribution becomes exponentially difficult to sample as $\sigma_0$ increases: negative $x_0$ values produce exponentially small spreads in $x_{1:9}$, while positive values yield exponentially large spreads. We examine cases with $\sigma_0=1$ and $\sigma_0=3$.

Following~\citet{zhang2022path}, we parameterize the control as
\begin{equation*}
u^\theta_t (x) = \text{NN}_1^\theta (t,x) + \text{NN}_2^\theta (t) \times \nabla \log \rho (x),
\end{equation*}
where $\nabla \log \rho (x)$ is the score of the funnel distribution density. For each network $\text{NN}_1^\theta$ and $\text{NN}_2^\theta$, we encode the scalar $t$ into $128$ dimensions using Fourier positional encoding, process it through two fully connected layers with $64$ hidden units, and separately process $x$ through a two-layer MLP to obtain $64$-dimensional features. The concatenated features feed into a three-layer network for the final output. To increase difficulty, the last linear layers of both networks are initialized to zero, yielding $u^\theta_t (x) = 0$ initially. Optimization uses Adam \citep{kingma2014adam} with learning rate $5\cdot 10^{-3}$.

\begin{table}
\centering
\setlength{\tabcolsep}{2.5pt}
\begin{tabular}{@{} l r r r r@{}}\toprule
\textbf{Algorithm}  
& \textbf{MMD} $\downarrow$ \\
\cmidrule(r){1-2}

Our Method & 
$\mathbf{0.043 \pm 0.001}$ \\

PIS~\citep{zhang2022path} & 
$\mathbf{0.048 \pm 0.001}$\\
FAB~\citep{midgley2022flow}& $\mathbf{0.032 \pm 0.000}$ \\ 
GMMVI~\citep{arenz2022unified}& $\mathbf{0.031 \pm 0.000}$ \\ 
DDS~\citep{vargas2023denoising} & $0.172 \pm 0.031$   \\ 
AFT~\citep{arbel2021annealed}       & $0.159 \pm 0.010$ \\
CRAFT~\citep{arbel2021annealed}         & $0.115 \pm 0.003$ \\ 
CMCD-KL~\citep{nusken2024transport} & $0.095 \pm 0.003$ \\ 
NETS-AM & $\mathbf{0.041 \pm 0.001}$ \\ ~\citep{albergo2025netsnonequilibriumtransportsampler} \\ 
\bottomrule
\end{tabular}
\caption{\textbf{Funnel Distribution Example}: Performance of our method measured by MMD (Maximum Discrepancy Distance) from the true distribution. Benchmarking is quoted from comparative results of \citet{blessing2024} and \citet{albergo2025netsnonequilibriumtransportsampler} for reproducibility. Following Blessing et al. (2024), we compute the MMD
between 10000 samples from the model and 10000 samples from the target and compare to the numbers reported by \citet{blessing2024} and \citet{albergo2025netsnonequilibriumtransportsampler}. }
\label{table:MMD_funnel}
\end{table}

\begin{figure}[!hbt]
    \centering
\includegraphics[width=0.32\linewidth]{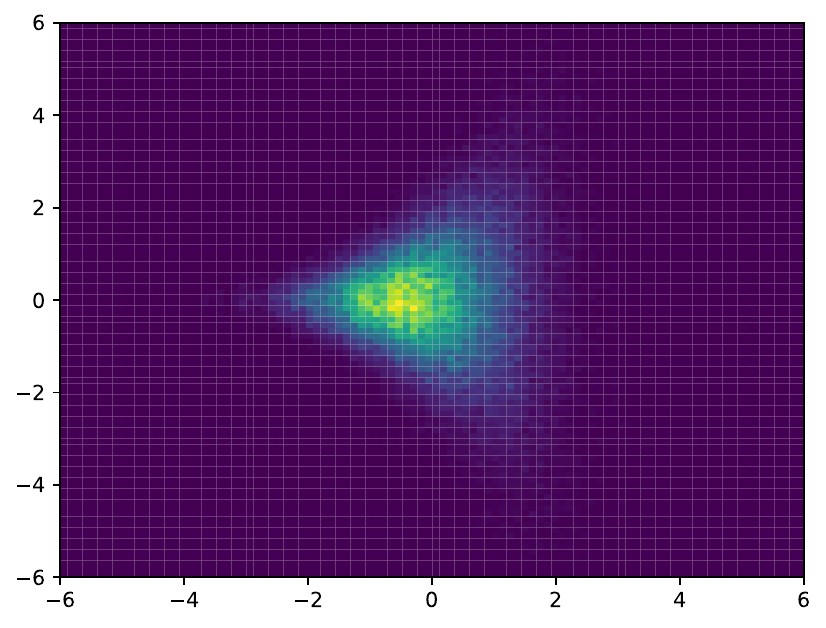}
\includegraphics[width=0.32\linewidth]{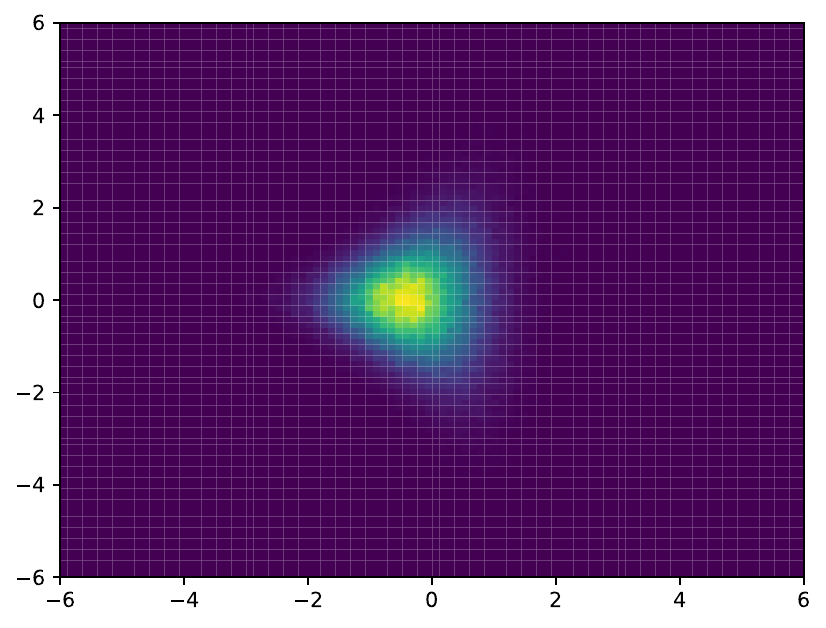}
\includegraphics[width=0.32\linewidth]{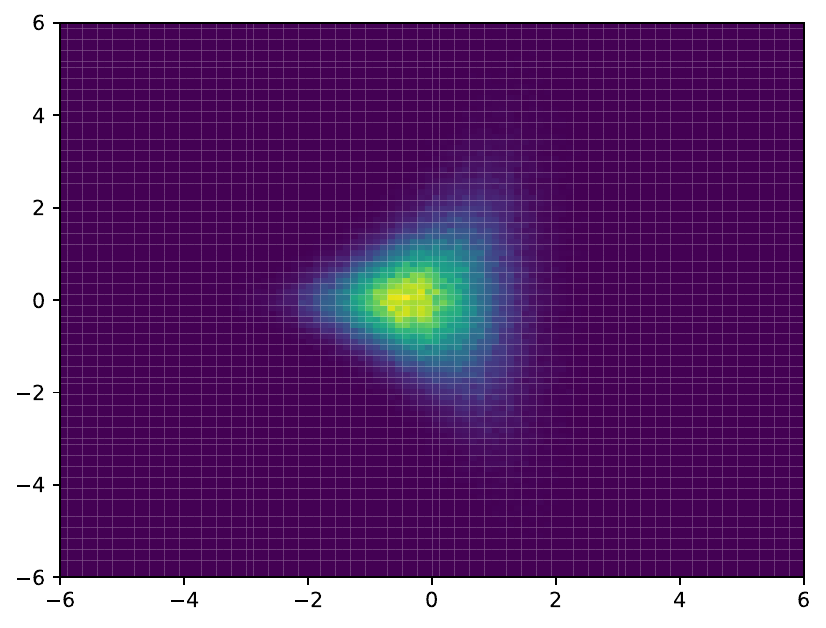}

\caption{{\bf Funnel distribution example ($\sigma = 1$):} The samples from the funnel distribution (left panel), Path Integral Samplers (middle panel) and our method (right panel). To plot the samples of the ten-dimensional funnel distribution in 2D, we use the independence of its coordinates $\{x_1,\cdots, x_9\}$, squeeze these nine dimensions into one coordinate, and keep the first dimension $x_0$. }
\label{fig:funnel_samples}
\end{figure}

\subsection{Fine-tuning on the $\varphi^4$ model}

We sample the $\varphi^4$ model in $d = 2$ spacetime dimensions, a statistical lattice field theory where field configurations $\varphi \in \mathbb{R}^{L \times L}$ represent the lattice state ($L$ denotes spatiotemporal extent). This model poses sampling challenges due to its phase transition from disorder to full order, during which neighboring sites develop strong correlations in sign and magnitude \cite{vierhaus, albergo2019}. Using the framework from Sec.~\ref{sec:finetuning}, we fine-tune a Gaussian distribution to match the fully ordered $\varphi^4$ model. 

The $\varphi^4$ model is specified by the following probability density function (PDF)
\begin{equation}
\label{eq:phi4:pdf}
    \rho(\varphi) = Z^{-1} e^{-E(\varphi)}
\end{equation}
where $Z = \int_{\R^{L\times L}} e^{-E(\varphi)} d\varphi$ is a normalization constant and $E$ is an energy function defined as
\begin{equation}
    \label{eq:energy:phi4}
    \begin{aligned}
        E(\varphi)  = & \frac12\alpha \sum_{a\sim b} |\varphi(a) - \varphi(b)|^2 +  \frac12\beta  \sum_{a} |\varphi(a)|^2  +  \frac14\gamma \sum_{a} |\varphi(a)|^4.
    \end{aligned}
\end{equation}
where $a,b\in[0,\ldots, L-1]^2$ denote the discrete positions on a $2$-dimensional lattice of size $L\times L$, $a\sim b$ denotes neighboring sites on the lattice, and we assume  periodic boundary conditions; $\alpha>0$ , $\beta \in \R$ and $\gamma >0$ are parameters. We will sample PDF~\eqref{eq:phi4:pdf} by fine-tuning a reference SDE whose time $t=1$ solutions sample the Gaussian PDF
\begin{equation*}
\rho_0(\varphi) = Z_0^{-1} e^{-E_0 (\varphi)},
\end{equation*}
where $Z_0 = \int_{\R^{L\times L}} e^{-E_0 (\varphi)} d\varphi$ and
\begin{equation}
\label{eq:energy:phi4:0}
E_0(\varphi) = \frac12\alpha \sum_{a\sim b} |\varphi(a) - \varphi(b)|^2 +  \frac12\beta_0  \sum_{a} |\varphi(a)|^2,
\end{equation}
with $\alpha>0$ as in \eqref{eq:energy:phi4} and $\beta_0>0$. Writing $E_0(\varphi)= \frac12 \varphi^T C^{-1} \varphi$ shows that $\rho_0(\varphi)$ is a zero-mean Gaussian with covariance $C$ (see Appendix~\ref{app:phi_4}). Thus, solutions to
\begin{equation}
\label{eq:ref_process_phi_4}
d \varphi^0_t =  C^{-1/2} d W_t, \qquad \varphi^0_{t=0} = 0
\end{equation}
satisfy $\varphi^0_{t = 1} \sim \rho_0 (\varphi)$, and we take it as reference process to fine-tune.

Specifically, we consider the solution $u$ to the optimal control problem with the objective
\begin{equation}
\label{eq:objective_phi_4}
    \min_{ u} \E \left[ \frac{1}{2}  
    \int_0^1  \sum_a | u(t,a,\varphi^u_t )|^2 dt +  U( \varphi^u_T  ) \right],
\end{equation}
where $ \varphi^u_t(a)$ solves the SDE
\begin{equation}
\label{eq:control_process_phi_4}
     d  \varphi^u_t(a)  = C^{-1}(a) u (t,a,\varphi^u_t) dt  + C^{-1/2}(a) d W_t(a)
\end{equation}
with $\varphi^u_{t=0} (a)  = 0, U(\varphi^u_T) = E(\varphi^u_T) - E_0(\varphi^u_T)$ and $T = 1$. By Proposition~\ref{th:follmer}, if we use the optimal control in~\eqref{eq:control_process_phi_4_fourier}, we have that $\varphi^u_{t=1} (a )$ samples the PDF~\eqref{eq:phi4:pdf}.

\begin{figure}[t]
    \centering
\includegraphics[width=0.4\linewidth]{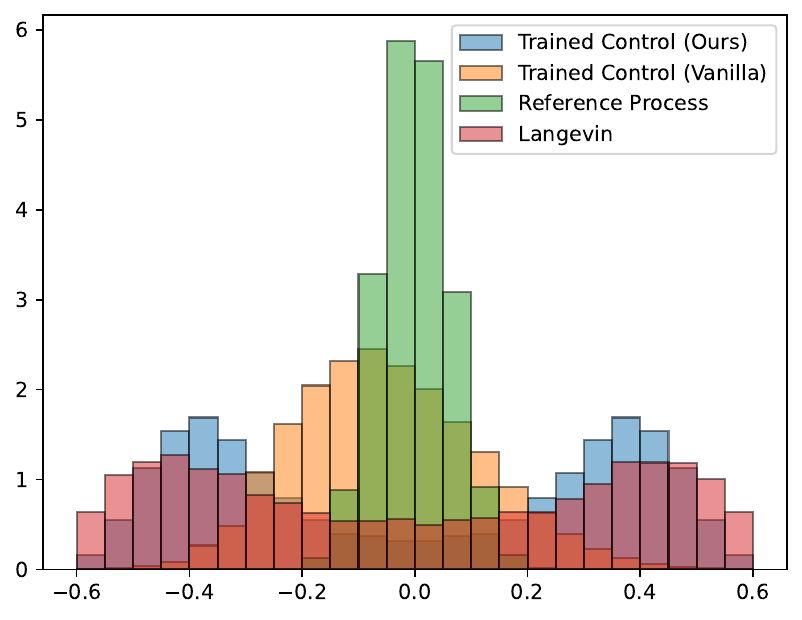}
\includegraphics[width=0.4\linewidth]{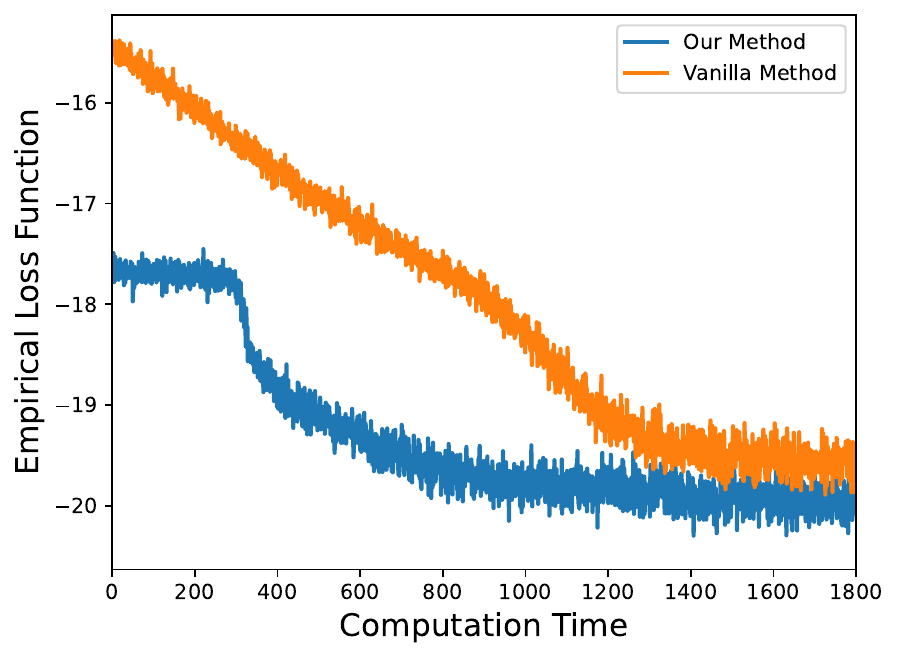}
\caption{\textbf{Top Panel}: Histograms of the average magnetization of $10000$ lattice configurations, sampled with the trained control of our method, the vanilla method, Langevin dynamics with $E(\varphi)$, and the reference PDF $\rho_0(\varphi)$. We use samples obtained by running Langevin dynamics with the target potential $E(\varphi)$ as the ground-truth. (See Appendix~\ref{app:phi_4} for more details about how to sample with Langevin dynamics.) Note that the results of our method is closer to the Langevin target than the vanilla method, which has adopts a much smaller model and fails to capture the statistics. \textbf{Bottom Panel}: The empirical loss function over GPU compute time. Our method enables a much larger model under the same memory budget and therefore presents a much better learning curve as compared to the vanilla method .}
\label{fig:phi_4}
\end{figure}

\paragraph{Numerical Results:} We tested various setups and report the results of some of them here. With the same setup as described in~\citet{albergo2025netsnonequilibriumtransportsampler}, we choose $L = 16$, $\alpha = 2.0,\beta = -2.0, \gamma = 3.2$ so that the lattice system is at the phase transition and in the ordered phase. 

Compared to the experiments done in Sec.~\ref{sec:unnorm}, this fine-tuning task has a much higher dimensionality and therefore requires more expressive neural networks for effective training. For such situations, the vanilla method no longer scales efficiently. We performed experiments comparing our method to the vanilla method \textit{while maintaining the same memory constraints}. Our method significantly reduces memory usage compared to the vanilla approach, which relies on storing the entire computational graph during SDE integration. This reduction allows our method to train models with a much larger number of parameters, resulting in superior performance. The results of these experiments are summarized in Figure~\ref{fig:phi_4}.  

\section{Conclusion}
\label{sec:conclu}

We have introduced a simulation-free on-policy approach to SOC problems: we simulate trajectories using the actual control but detach this control for the computational graph when computing the gradient of the objective. This yields an efficient, scalable method for training deep neural networks to learn feedback controls, outperforming traditional vanilla approaches. We demonstrated its effectiveness across SOC and sampling applications, including F\"ollmer processes and diffusion model fine-tuning.

\section*{Reproducibility Statement}
All experiments done in this work rely on simply feed-forward neural networks, and can be done locally on a single GPU. Details for the network sizes are given in each experimental subsection. 

\section*{Acknowledgements}
We thank Michael Albergo, Joan Bruna, and Carles Domingo-Enrich for helpful discussions.  EVE is supported by the National Science Foundation under Awards DMR1420073, DMS-2012510, and DMS-2134216, by the Simons Collaboration on Wave Turbulence, Grant No. 617006, and by a Vannevar Bush Faculty Fellowship.

\bibliography{iclr2025_conference}
\bibliographystyle{iclr2025_conference}

%%%%%%%%%%%%%%%%%%%%%%%%%%%%%%%%%%%%%%%%%%%%%%%%%%%%%%%%%%%%%%%%%%%%%%%%%%%%%%%
%%%%%%%%%%%%%%%%%%%%%%%%%%%%%%%%%%%%%%%%%%%%%%%%%%%%%%%%%%%%%%%%%%%%%%%%%%%%%%%
% APPENDIX
%%%%%%%%%%%%%%%%%%%%%%%%%%%%%%%%%%%%%%%%%%%%%%%%%%%%%%%%%%%%%%%%%%%%%%%%%%%%%%%
%%%%%%%%%%%%%%%%%%%%%%%%%%%%%%%%%%%%%%%%%%%%%%%%%%%%%%%%%%%%%%%%%%%%%%%%%%%%%%%
\appendix

\section{Proofs of Propositions~\ref{th:grad}, \ref{th:fine:tune} and \ref{th:fine:tune:2}}
\label{app:finetuning}

\begin{proof}[Proof of Proposition~\ref{th:grad}]
    Equation~\eqref{eq:grad} follows from~\eqref{eq:obj:G} by a direct calculation in which we first evaluate the gradient of the objective $L(\theta) = J(u^\theta)$ with the fixed reference control $v$, which gives:
    \begin{equation}
\label{eq:grad:1}
\begin{aligned}
\partial_\theta L(\theta)
    & = \EE\left(\left[ \int_0^T u_t^\theta(X^{v}_t) \cdot \partial_\theta u_t^\theta(X^{v}_t) dt\right] M(u^\theta,v)\right)
    \\
    &+ \EE\left[ \left(\int_{0}^T \left(\tfrac{1}{2} | u^\theta_t(X_t^\theta)|^2 + f_t(X_t^\theta) \right)  dt + g(X_T^v)\right) \right. \\
    & \quad \times \left. \left(\int_0^T \partial_\theta u_t^\theta(X^{v}_t)\cdot dW_t + 
    \int_0^T (v_t(X_t^v)-u^\theta(X_t^v)) \cdot \partial_\theta u_t^\theta(X^{v}_t)\cdot dW_t \right) M(u^\theta,v)\right].
\end{aligned}
\end{equation}
Because \eqref{eq:grad:1} holds for any $v$, we can now evaluate it at $v = u^\theta$. Since $M(u^\theta,u^\theta)=1$, this gives~\eqref{eq:grad}. 
\end{proof}

\begin{proof}[Proof of Proposition~\ref{th:fine:tune}] It is well-known~\citep{leonard2014survey,chen2014relation} that the SOC problem specified in the proposition can be cast into solving the pair of partial differential equations
\begin{align}
    \label{eq:obj:pde:mu}
        \partial_t \mu_t &= - \nabla\cdot (b_t \mu_t) - \nabla \cdot(D_t \nabla \phi_t \mu_t) + \tfrac12 \nabla \cdot(D_t \nabla \mu_t), && \mu_0 = \delta_0\\
        \label{eq:obj:pde:phi}
        \partial_t \phi_t &= - b_t \cdot \nabla\phi_t - \tfrac12 \nabla \phi_t \cdot D_t \nabla \phi_t - \tfrac12 \nabla \cdot(D_t \nabla \phi_t),  && \phi_T = r
\end{align}
where $D_t = D_t^T= \sigma_t \sigma_t^T$,  $\mu_t$ is the distribution of $X^u_t$, and the potential $\phi_t$ gives the optimal control via $u_t = \sigma_t^T \nabla \phi_t$. We also know that the distribution $\nu_t$ of $Y_t$ solves the Fokker-Planck equation
\begin{equation}
    \label{eq:y:pde}
        \partial_t \nu_t = - \nabla\cdot (b_t \nu_t)  + \tfrac12 \nabla \cdot(D_t \nabla \nu_t), \qquad \nu_0 = \delta_0
\end{equation}
We will prove that $X^u_T \sim \mu$ by establishing that $\mu_t(dx) = Z^{-1}e^{\phi_t(x)} \nu_t(dx)$ since this will imply that $\mu_T = \mu$ because $\nu_T = \nu$ by definition and $\phi_T = r$ by construction, so that $\mu_T = Z^{-1}e^{\phi_T} \nu_T = Z^{-1}e^{r}\nu = \mu$. Let $\hat \mu_t(dx) = Z^{-1}e^{\phi_t(x)} \nu_t(dx)$. Then we have 
\begin{equation}
    \label{eq:obj:pde:mu:nu}
    \begin{aligned}
        \partial_t \hat\mu_t &= Z^{-1}\left(\partial_t \phi_t e^{\phi_t} \nu_t + e^{\phi_t} \partial_t \nu_t\right),\\
        - \nabla\cdot (b_t \hat\mu_t) &= Z^{-1}\left(- e^{\phi_t} \nabla\cdot (b_t \nu_t) - b_t\cdot \nabla \phi_t e^{\phi_t} \nu_t\right),\\ 
        - \nabla \cdot(D_t \nabla \phi_t \hat\mu_t) & = Z^{-1}\left(- e^{\phi_t}  \nabla \cdot(D_t \nabla \phi_t \nu_t) - e^{\phi_t}\nabla \phi_t \cdot D_t \nabla \phi_t  \nu_t\right)\\
        & = Z^{-1}\left(- e^{\phi_t}  \nabla \cdot(D_t \nabla \phi_t ) \nu_t -  e^{\phi_t}  \nabla \phi_t \cdot D_t \nabla \nu_t- e^{\phi_t}\nabla \phi_t \cdot D_t \nabla \phi_t  \nu_t\right),\\
        \tfrac12 \nabla \cdot(D_t \nabla \hat\mu_t) & = Z^{-1}\left(  \tfrac12 e^{\phi_t} \nabla \cdot(D_t \nabla \nu_t) + e^{\phi_t}\nabla \phi_t \cdot D_t \nabla \nu_t \right.\\
        & \left. \quad + \tfrac12e^{\phi_t} \nabla \phi_t \cdot D_t \nabla \phi_t \nu_t + \tfrac12 e^{\phi_t} \nabla \cdot(D_t \nabla \phi_t) \nu_t \right).
\end{aligned}
\end{equation}
Inserting these expressions in the left-hand side and the right-hand side of~\eqref{eq:obj:pde:mu} and using~\eqref{eq:y:pde}, several terms cancel we are left with 
\begin{equation}
    \label{eq:pde:mu:nu:2}
    \partial_t \phi_t e^{\phi_t} \nu_t  = - b_t\cdot \nabla \phi_t e^{\phi_t} \nu_t  - \tfrac12\nabla \phi_t \cdot D_t \nabla \phi_t e^{\phi_t} \nu_t - \tfrac12 e^{\phi_t}  \nabla \cdot(D_t \nabla \phi_t ) \nu_t.
\end{equation}
If we divide both sides of this equation by $e^{\phi_t} \nu_t$, we recover~\eqref{eq:obj:pde:phi}. This shows that $\hat\mu_t(dx) = Z^{-1}e^{\phi_t(x)} \nu_t(dx)$ is indeed a solution to the PDE~\eqref{eq:obj:pde:mu}. To show that it is the solution, it remains to establish that it satisfies the initial condition in~\eqref{eq:obj:pde:mu}. To this end, notice first that, since $\nu_0 = \delta_0$, we have 
\begin{equation}
    \label{eq:mu:at0}
    \hat \mu_0(dx) = Z^{-1} e^{\phi_0(x)} \nu_0(dx) = Z^{-1} e^{\phi_0(0)} \delta_0(dx)
\end{equation}
Second, since $\hat \mu_t$ satisfies~\eqref{eq:obj:pde:mu}, we must have $\int_{\R^d} \mu_t(dx) = 1$ for all $t\in [0,1]$. As a result, we conclude that $e^{\phi_0(0)} = Z$, which  means that $\hat \mu_0 = \delta_0 = \mu_0$. Since the solution pair $(\mu_t,\phi_t)$ to~\eqref{eq:obj:pde:mu}-\eqref{eq:obj:pde:phi}  is unique, we must have $\mu_t = \hat \mu_t = Z^{-1} e^{\phi_t} \nu_t $ and hence $\mu_T = Z^{-1} e^{\phi_T} \nu_T =  Z^{-1} e^{r} \nu$.
\end{proof}

Note that it is key that $\mu_0 = \nu_0 = \delta_0$ (more generally $\delta_{x_0}$ for some $x_0\in \R^d$). If the base distribution used to generate initial data in the SDEs~\eqref{eq:sde} and~\eqref{eq:ref:Y} are not atomic at $x=0$, the statement of Proposition~\ref{th:fine:tune} does not hold anymore, because the second equality in~\eqref{eq:mu:at0} fails. That is, our framework only allows to fine-tune generative models that use a Dirac delta distribution as base distribution.

\begin{proof}[Proof of Proposition~\ref{th:fine:tune:2}]
    By direct application of Girsanov theorem, we have
\begin{equation}
    \label{eq:gir:proof}
    \E_{X^u} \left[h(X^u_T) M_r(u)\right] = \E_{Y}\left[ h(Y_T)e^{r(Y_T)}\right] = \int_{\R^d} h(x) e^{r(x)} \nu(dx),
\end{equation}
where $Y_t$ solves~\eqref{eq:ref:Y} and we used $Y_{T} \sim \nu$ to get the second equality. Multiplying both sides of~\eqref{eq:gir:proof} by $Z^{-1}$ we deduce 
\begin{equation}
    \label{eq:gir:proof2}
    Z^{-1} \E_{X^u} \left[ h(X^u_T) M_r(u)\right] = Z^{-1} \int_{\R^d}  h(x) e^{r(x)} \nu(dx)  = \int_{\R^d}  h(x) \mu(dx),
\end{equation}
which gives the first equation in~\eqref{eq:expect:ft}.
Setting $h=1$ in \eqref{eq:gir:proof} we deduce that
\begin{equation}
    \label{eq:gir:proof2b}
    \E_{X^u} \left[ M_r(u)\right] = \int_{\R^d}  e^{r(x)} \nu(dx)  = Z
\end{equation}
which gives the second equation in~\cref{eq:expect:ft}.  To establish that  $M_r(u)=Z$ \textit{iff} $u$ is the optimal control minimizing the SOC problem specified in Proposition~\ref{th:fine:tune}, notice that $X^u_T\sim \mu$ \textit{iff} $u$ is this optimal control.  Assuming that this is the case, the first equation in~\cref{eq:expect:ft} implies that
\begin{equation}
    \label{eq:gir:proof3}
    \E_{X^u}\left[ h(X^u_T)\right] = Z^{-1} \E_{X^u} \left[ h(X^u_T) M_r(u)\right] 
\end{equation}
for all suitable test function~$h$. This can only hold if $M_r(u) = Z$.
\end{proof}

\section{Fourier Representation of the $\varphi^4$ Model}
\label{app:phi_4}

We define the discrete Fourier transform with the following:
\begin{equation}
    \label{eq:fourier}
    \hat \varphi (k ) = L^{-d/2} \sum_a e^{2i\pi k\cdot a /L } \varphi(a)\quad \Leftrightarrow \quad \varphi (a ) = L^{-d/2} \sum_k e^{-2i\pi k\cdot a /L } \hat \varphi(k)
\end{equation}
where $a,k \in [0,\ldots, L-1]^d$, we can write the energy \eqref{eq:energy:phi4:0} as
\begin{equation}
    \label{eq:energy:phi4:0:F}
    E_0(\varphi) = \hat E_0(\hat \varphi) \equiv\frac12\sum_{k} \hat M(k)|\hat \varphi(k)|^2, \qquad \hat M(k) = 2\alpha \left(d-\sum_{\hat e }\cos (2\pi k\cdot \hat e /L)\right) +  \beta_0,
\end{equation}
where $\hat e$ denotes the $d$ basis vectors on the lattice.   Therefore $\rho_0$ is a Gaussian density with covariance
\begin{equation}
\label{eq:cov:phi4:0:F}
    \int_{\R^{L^d}} \varphi(a) \varphi(b) \rho_0(\phi) d\phi = C(a-b), \qquad C(a) = L^{-d/2} \sum_k e^{-2i\pi k\cdot a /L } \hat M^{-1}(k).
\end{equation}

Let the inverse discrete Fourier transform defined by
\begin{equation}
    \label{eq:fourier:inv}
    \varphi^0_t (a ) = L^{-d/2} \sum_k e^{-2i\pi k\cdot a /L } \hat \varphi^0_t(k)
\end{equation}
where $\varphi^0_t(k)$ solves 
\begin{equation}
\label{eq:ref_process_phi_4:2}
    d\hat \varphi^0_t(k) = \hat M^{-1/2}(k) d\hat W_t(k), \qquad \hat \varphi^0_{t=0}(k) = 0
\end{equation}
in which $\hat W$ is the Fourier transform of a Wiener process $W$ with covariance $\E[W_t(a)dW_s(b) ]= \delta_{a,b} \min(t,s)$. Then $\varphi^0_{t=1} \sim \rho_0$. Let also
\begin{equation}
    \label{eq:U:phi4}
    \begin{aligned}
        & U(\varphi) = \frac12 (\beta-\beta_0) \sum_a |\varphi(a)|^2 + \frac14 \gamma \sum_a  |\varphi(a)|^4\\
        = \ & \hat U(\hat \varphi) = \frac12 (\beta-\beta_0) \sum_k |\hat\varphi(k)|^2 + \frac14 \gamma \sum_a  \left|L^{-d/2} \sum_k e^{-2i\pi k\cdot a /L } \hat \varphi(k)\right|^4
    \end{aligned} 
\end{equation}
where $\varphi$ and $\hat\varphi$ are Fourier transform pairs as defined in \eqref{eq:fourier}: the last term can be implemented via $\sum_a (\text{ifft}(\hat \varphi))^4(a)$.
Then, we can derive the Fourier representation of the optimal control problem with objective~\eqref{eq:objective_phi_4} and controlled process~\eqref{eq:control_process_phi_4}: 

\begin{equation}
\label{eq:objective_phi_4_fourier}
    \min_{\hat u} \E [ \frac{1}{2}  
    \int_0^1  \sum_k |\hat u(t,k,\hat\varphi^u_t )|^2 dt + \hat U(\hat \varphi^u_T  )]
\end{equation}
where $\hat \varphi^u_t(k)$ solves the controlled process
\begin{equation}
\label{eq:control_process_phi_4_fourier}
    d \hat \varphi^u_t(k)  = \hat M^{-1}(k)\hat u (t,k,\hat\varphi^u_t) dt  +\hat M^{-1/2}(k) d\hat W_t(k), \qquad \hat \varphi^u_{t=0} (k)  = 0
\end{equation}
Then by Proposition~\ref{th:follmer}, if we use the optimal control in~\eqref{eq:control_process_phi_4}, we have that 
\begin{equation}
    \label{eq:fourier:2}
    \varphi^u_{t=1} (a ) = L^{-d/2} \sum_k e^{-2i\pi k\cdot a /L } \hat \varphi^u_{t=1}(k)
\end{equation}
sample the PDF~\eqref{eq:phi4:pdf}.

\paragraph{Sampling using the Langevin SDE:} To obtain the ground-truth samples from the $\varphi^4$ model, one option is to use the SDE
\begin{equation}
\label{eq:clang_phi_4}
    d \hat \varphi_t(k)  = -\hat M(k) \hat \varphi_t(k) dt - (\beta-\beta_0) \hat \varphi_t(k) dt - \gamma \widehat{\varphi_t^3} (k) dt+ \sqrt{2} d\hat W_t(k).
\end{equation}
where we denote
\begin{equation}
    \label{eq:pseudo}
    \widehat{\varphi_t^3} (k) = L^{-d/2} \sum_a e^{2i\pi k\cdot a /L }\left(L^{-d/2} \sum_k e^{-2i\pi k\cdot a /L } \hat \varphi_t(k)\right)^3 
\end{equation}
which can be implemented via $\text{fft}( (\text{ifft}(\hat\varphi_t))^3)$. This SDE may be quite stiff, however, a problem that can be alleviated by changing the mobility and using instead
\begin{equation}
\label{eq:clang_phi_4:2}
    d \hat \varphi_t(k)  = -\hat \varphi_t(k) dt - (\beta-\beta_0) \hat M^{-1} (k) \hat \varphi_t(k) dt  - \gamma \hat M^{-1} (k)\widehat{\varphi_t^3} (k) dt + \sqrt{2} \hat M^{-1/2}(k) d\hat W_t(k).
\end{equation}
The discretized version of this equation reads
\begin{equation}
\label{eq:clang_phi4:d}
\begin{aligned}
    \hat \varphi_{t_{n+1}}(k)  & = \hat \varphi_{t_n}(k) - \mathit{\Delta}{t_n}  \left( \hat \varphi_{t_n}(k) + (\beta-\beta_0) \hat M^{-1} (k) \hat \varphi_{t_n}(k) +\gamma \hat M^{-1} (k)\widehat{\varphi_{t_n}^3} (k) \right) \\
    &\quad + \sqrt{2\mathit{\Delta}{t_n}}\hat M^{-1/2}(k)  \hat\eta_n(k),     
\end{aligned}
\end{equation}
where $\hat \eta_n$ is the Fourier transform of $\eta_n \sim N(0,\text{Id})$.

\section{Additional Numerical Examples}
\label{app:add_expeiments}

\subsection{Linear Ornstein-Uhlenbeck Example}
\label{sec:ou:lin}

We consider the SOC problem in~\eqref{eq:obj} with $f = 0$, $g(x) = \gamma\cdot x$, $\sigma_t= \sigma = cst$, $\lambda = 1$ and $ b_t(x) = Ax$, where $\gamma\in \R^d$ and $\sigma,A\in \R^d\times \R^d$. This example was proposed by~\citet{nüsken2023solvinghighdimensionalhamiltonjacobibellmanpdes} and its optimal control can be calculated analytically:
\begin{equation*}
    u^\ast_t(x) = - \sigma_0^{\top} \exp(A^{\top} (T-t)) \gamma.
\end{equation*}
We set $d = 20$ with initial samples $X_0 \sim \mathcal{N} (0, \frac12\text{Id})$. The control $u_t(x)$ is parameterized using a fully connected MLP with 4 layers of 128 hidden dimensions, initialized using PyTorch defaults. Optimization uses Adam \citep{kingma2014adam} with learning rate $3\cdot 10^{-4}$ and cosine annealing. For comparison, we implement the vanilla method requiring SDE differentiation under identical conditions. Performance is evaluated using the squared $L^2$ error between learned and true controls.
\begin{equation}
\label{eq:l2error}
    E =   \int_{0}^{T} \mathbb{E}\left[\left|u_t(X^{\ast}_t) - u^\ast_t(X^{\ast}_t) \right|^2 \right] dt
\end{equation}
where $X^{\ast}$ is generated with the optimal control $u^\ast$ and the expectation is estimated via Monte-Carlo sampling over $256$ trajectories.

The numerical results are shown in Figure~\ref{fig:linear_OU} . Compared with the vanilla method, our method achieves comparable  accuracy faster and at lower memory cost (see Table~\ref{table:linear_OU} for a detailed comparison in terms of memory cost and computational time). % and $X^\ast_0 \sim \mathcal{N} (0, \frac12\text{Id})$.
\begin{figure}[t]
    \centering
\includegraphics[width=0.48\linewidth]{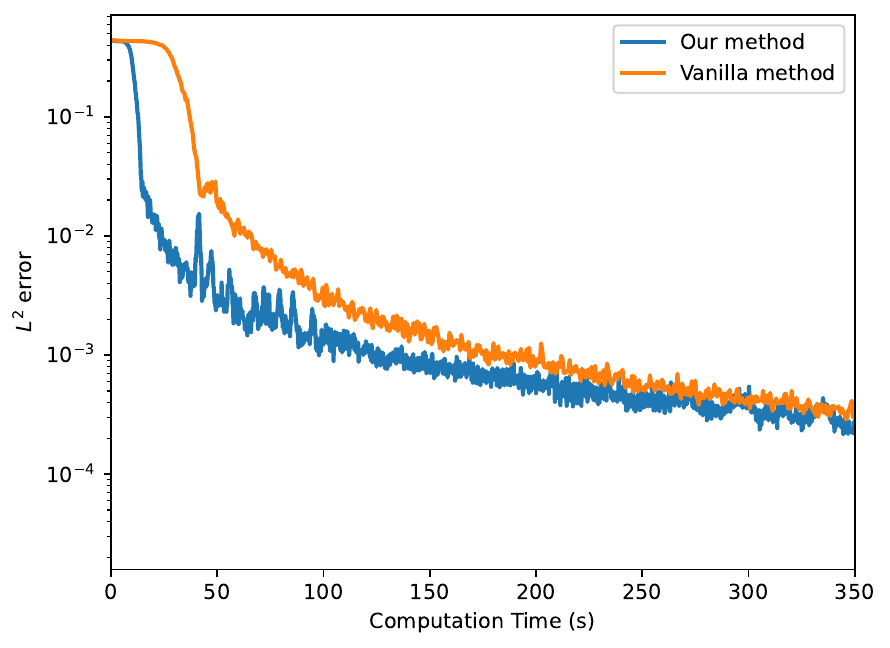}
\includegraphics[width=0.48\linewidth]{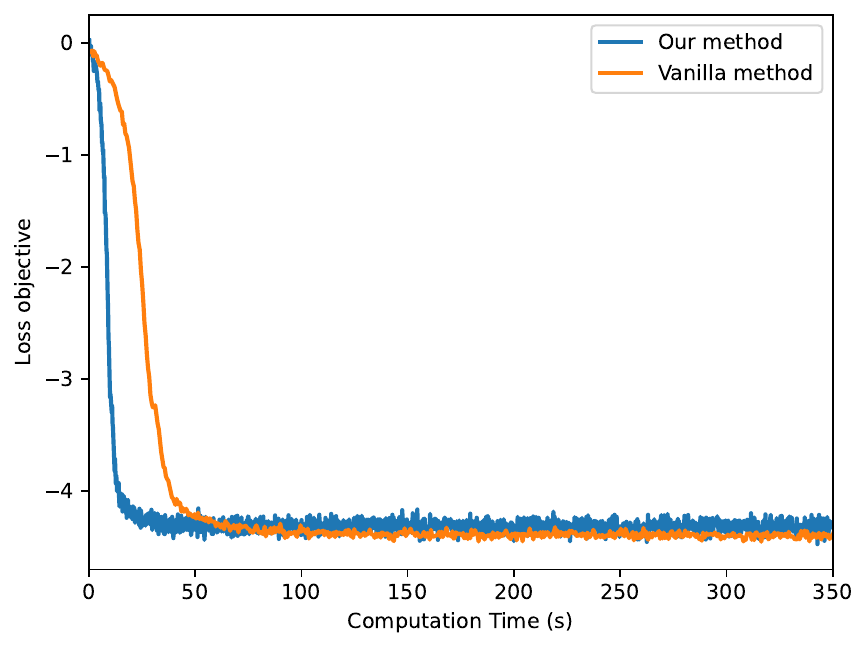}
\caption{\textbf{Linear Ornstein-Uhlenbeck Example}: Our method outperforms the vanilla method in terms of convergence rate measured by the squared $L^2$ error  (top panel) and the training loss (bottom panel). }
\label{fig:linear_OU}
\end{figure}

\begin{table}
\centering
\setlength{\tabcolsep}{2.5pt}
% \resizebox{\linewidth}{!}{%
% \small
\begin{tabular}{@{} l r r r r@{}}\toprule
Model  
& \tcen{Memory Cost (GB)} & \tcen{Back-Prop Runtime (s)}\\
\cmidrule(r){1-1}\cmidrule(lr){2-2} \cmidrule(lr){3-3} 

Our Method & 
$\mathbf{0.962}$ {\tiny $\pm 0.001$} & 
$\mathbf{0.003}$ {\tiny $\pm 0.000$}\\

Vanilla Method & 
$2.590$ {\tiny $\pm 0.001$} & 
$0.177$ {\tiny $\pm 0.006$}\\
\bottomrule
\end{tabular}
\caption{\textbf{Linear Ornstein-Uhlenbeck Example}: Comparison between our method and the vanilla method in terms of the GPU memory usage and runtime for one back-propagation pass. Here, we use a mini-batch size of $5$k and $256$ time steps. }
\label{table:linear_OU}
\end{table}

\subsection{Quadratic Ornstein-Uhlenbeck Example}
\label{sec:ou:quad}

\begin{table}
\centering
\setlength{\tabcolsep}{2.5pt}
\begin{tabular}{@{} l r r r r@{}}\toprule
Model  
& \tcen{Memory Cost (GB)} & \tcen{Back-Prop Runtime (s)}\\
\cmidrule(r){1-1}\cmidrule(lr){2-2} \cmidrule(lr){3-3} 

Our Method & 
$\mathbf{1.260}$ {\tiny $\pm 0.001$} & 
$\mathbf{0.0034}$ {\tiny $\pm 0.0003$}\\

Vanilla Method & 
$3.590$ {\tiny $\pm 0.001$} & 
$0.195$ {\tiny $\pm 0.003$}\\
\bottomrule
\end{tabular}
\caption{\textbf{Quadratic Ornstein-Uhlenbeck Example}: Comparison between our method and the vanilla method in terms of the GPU memory usage and runtime for one back-propagation pass. Here, we use a mini-batch size of $512$ and $256$ time steps. }
\label{table:quadratic_OU}
\end{table}

\begin{figure}[t]
    \centering
\includegraphics[width=0.48\linewidth]{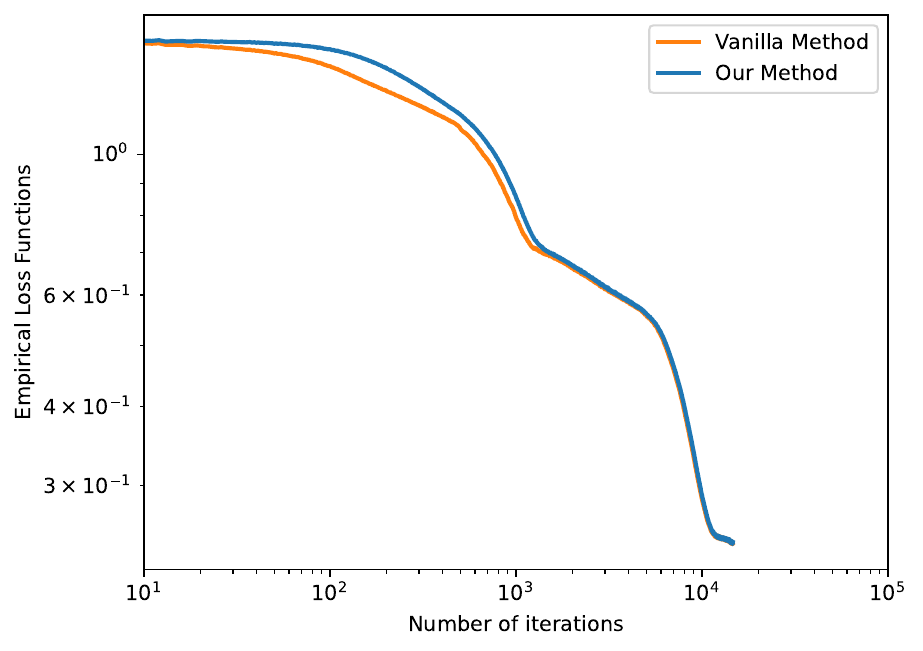}
\includegraphics[width=0.48\linewidth]{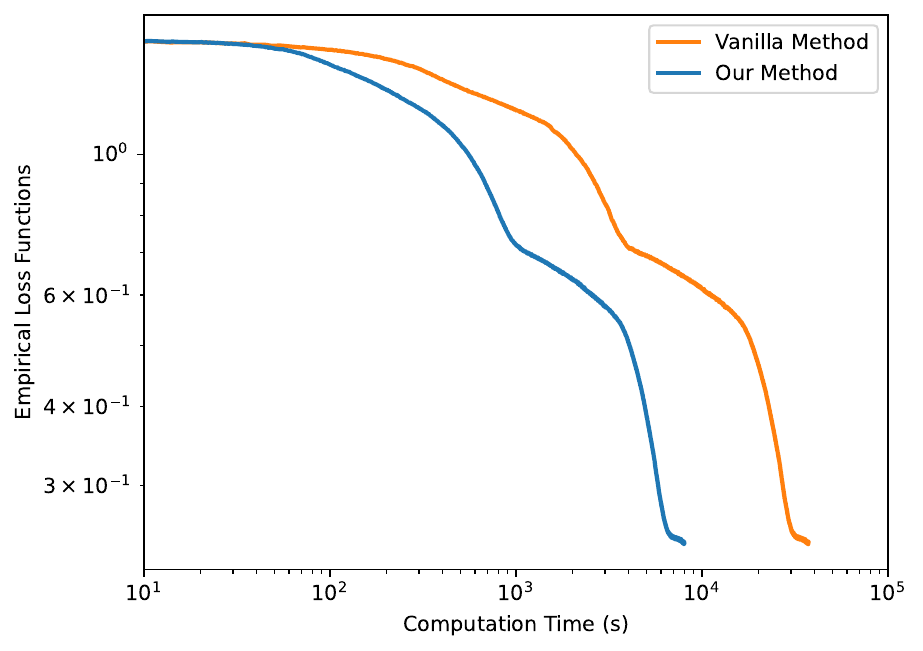}
\caption{\textbf{Quadratic Ornstein-Uhlenbeck Example}: $L^2_2$ error in terms of the number of training iterations (left panel) and the GPU compute time (right panel).} 
\label{fig:quadratic_OU_hard}
\end{figure}

Next, we consider a more complicated case where  the SOC objective includes a quadratic running cost: $f(x) = x^T P x$, $g(x) = x^T Q x$, $b_t(x) = Ax$, $\sigma_t = \sigma_0$, where $P,Q,A\in \R^d\times \R^d$. This type of SOC problems are often referred to as linear quadratic regulator (LQR) and they have closed-form analytical solution \citep[Chapter 7]{van2007stochastic}:
\begin{equation*}
    u^*_t(x) = - 2 \sigma_0^{\top} F_t x, 
\end{equation*}
where $F_t$ solves the Riccati equation
\begin{equation*}
    \frac{dF_t}{dt} + A^{\top} F_t + F_t A - 2 |\sigma_0^{\top} F_t|^2 + P = 0
\end{equation*}
with the final condition $F_T = Q$.
We consider this example investigated by \citet{domingo2023stochastic} with the following configuration:
\begin{itemize}[leftmargin=0.15in,noitemsep,topsep=0pt]
    \item[] $d=400$, $A = I$, $P = I$, $Q = 0.5 I$, $\sigma_0 = I$, $\lambda = 1$, $T= 10$, $X_0 \sim N(0,\tfrac12 \text{I})$.  
\end{itemize}
However, compared to \citet{domingo2023stochastic}, we scale the dimensions from $d = 20$ to $d = 400$ and time horizon from $T = $ to $T = 10$, significantly increasing the task complexity. The neural network parameterization and initialization for $u_t(x)$ follow Sec.~\ref{sec:ou:lin}. Table~\ref{table:quadratic_OU} compares memory consumption and computational cost between our method and the vanilla approach, while Figure~\ref{fig:quadratic_OU_hard} contrasts their $L^2$ accuracy and training time. With equal computation time, our method achieves better $L^2$ accuracy and maintains identical learning curves but with faster execution. Notably, our method demonstrates superior scalability as it converges in under 3 hours ($<10000$ seconds) while the vanilla approach requires over 14 hours ($>50000$ seconds).

\end{document}